%% file: main.tex
\documentclass{article}

\usepackage[preprint]{neurips_2019}
\input{preamble}

\usepackage{hyperref}
\usepackage{titletoc}

\dottedcontents{section}[2.9em]{\bfseries}{2.9em}{1pc}
\title{Lower Bounds for \\ Adversarially Robust PAC~Learning}

\author{Dimitrios I. Diochnos\thanks{Authors have contributed equally.}\\
University of Virginia\\
\texttt{diochnos@virginia.edu}
\And Saeed Mahloujifar\footnotemark[\value{footnote}]\\
University of Virginia\\
\texttt{saeed@virginia.edu}
\And Mohammad Mahmoody\thanks{Supported by NSF CAREER award CCF-1350939 and University of Virginia's SEAS Research Innovation Award.}\\
University of Virginia\\
\texttt{mohammad@virginia.edu}}

\begin{document}

\maketitle
\begin{abstract}
%In the Probably Approximately Correct (PAC) learning framework, an algorithm aims to produce a hypothesis $h$ that is    approximately  correct with high probability. In adversarial settings, an adversary might tamper with the testing or training phase. In the test phase, an evasion adversary perturbs the input $x$ to produce a close instance $\mal{x}$ with the goal of making $h$ make a mistake. In the training phase, a poisoning adversary might change a small fraction of the training data with the hope of making $h$ have undesired properties.

In this work, we initiate a formal study of probably approximately correct (PAC) learning under evasion attacks, where the adversary's goal is to \emph{misclassify} the adversarially perturbed sample point $\mal{x}$, i.e., $h(\mal{x})\neq c(\mal{x})$, where $c$ is the ground truth concept and $h$ is the learned hypothesis. Previous work on PAC learning of adversarial examples have all modeled adversarial examples as \emph{corrupted inputs} in which the goal of the adversary is to achieve $h(\mal{x}) \neq c(x)$, where $x$ is the original untampered instance. These two definitions of adversarial risk coincide for many natural distributions, such as images, but are incomparable in general.

We first prove that for many theoretically natural input spaces of high dimension $n$ (e.g., isotropic Gaussian in dimension $n$ under  $\ell_2$ perturbations), if the adversary is allowed to apply up to a \emph{sublinear}  $o(\norm{x})$ amount of perturbations on the test instances, PAC learning requires sample complexity that is \emph{exponential}  in $n$. This is in contrast with results proved using the corrupted-input framework, in which the   sample complexity of   robust learning is only polynomially more.

We then formalize \emph{hybrid} attacks in which the evasion attack is preceded by a  poisoning attack. This is perhaps reminiscent of ``trapdoor attacks'' in which a poisoning phase is involved as well, but the evasion phase here uses the error-region definition of risk that  aims at misclassifying the perturbed instances. In this case, we show  PAC learning is sometimes \emph{impossible} all together, even when it is possible without the attack (e.g., due to the bounded VC dimension).

%Adversarially robust classification deals with the task of learning models that are resistant to adversarial perturbations that happen during   test time (evasion attacks that find adversarial examples). An adversarially robust PAC learning would a learning rule that outputs still guarantees ``probably approximately correctness'' of the hypothesis even under such perturbations.

%In this work, we demonstrate barriers that exist for adversarially robust PAC learning under evasion attacks. In particular, we show that for a broad set of input instances (e.g., log-concave or uniform over the unit sphere) learning half spaces that are  robust to sublinear perturbations of the inputs require exponentially many samples. We do so by connecting the measure concentration of the input instances to their distribution-dependent sample complexity. We further show that poisoning attacks could (provably) bootstrap the power of evasion attacks to rule out their PAC learning (regardless of the number of samples). 

%Our lower bounds hold for the \emph{error-region} adversarial risk that models adversarial examples $\mal{x}$ as small perturbations of the original instance $x$ that misclassified. All previous work on PAC learning under evasion attacks used a different definition in which the goal of the classifier is to find the original label of $x$ rather than $\mal{x}$.

\end{abstract}
\bigskip
\setcounter{tocdepth}{1}

\tableofcontents
\clearpage
\input{intro}
\input{NewDefs}

\input{results}

\input{extensions.tex}

\input{Conclusion.tex}

% {
%\clearpage
 %\bibliography{biblio/adv_exps,biblio/references}

\input{Biblio.tex}
 \bibliographystyle{icml2019}
% }

%\clearpage
%\onecolumn

%\appendix
%\input{appendix}
\end{document}

%% file: preamble.tex
\usepackage[utf8]{inputenc} % allow utf-8 input
\usepackage[T1]{fontenc}    % use 8-bit T1 fonts
\usepackage{hyperref}       % hyperlinks
\usepackage{url}            % simple URL typesetting
\usepackage{booktabs}       % professional-quality tables
\usepackage{amsfonts}       % blackboard math symbols
\usepackage{nicefrac}       % compact symbols for 1/2, etc.
\usepackage{microtype}      % microtypography
\usepackage{nicefrac}
\usepackage{xspace}
\usepackage{times}
\usepackage{color}

\usepackage{paralist} % needed for asparaitem for both normal article class and the NIPS specific template
\usepackage{psfrag}
\usepackage{dsfont}
\usepackage{paralist}

\definecolor{RoyalBlue}{cmyk}{1, 0.50, 0, 0}
\definecolor{ForestGreen}{cmyk}{0.864, 0.0, 0.429, 0.396}
\definecolor{Brown}{cmyk}{0.0,0.692,0.925,0.529}

\usepackage{todonotes}

\usepackage[english, ruled, linesnumbered, vlined]{algorithm2e}

\usepackage{amssymb,amsmath,amsthm}

\usepackage{subfigure}
\usepackage[makeroom]{cancel}
\usepackage{thmtools,thm-restate}
\usepackage{comment}
\usepackage{multirow}

\input{macros}

\definecolor{carnelian}{rgb}{0.7, 0.11, 0.11}

\usepackage{todonotes}
\newcommand{\ifcomments}{\iffalse}

%% file: macros.tex
\usepackage{centernot}

\newcommand{\AER}{\mathcal{AE}}

\newcommand{\XX}{\ensuremath{\mathcal{X}}\xspace} % X
\newcommand{\YY}{\ensuremath{\mathcal{Y}}\xspace} % X
\newcommand{\loss}{\ell oss}

\newcommand{\Adv}{\mathsf{A}}
\newcommand{\adv}{\Adv}

\newcommand{\advC}{\mathcal{A}}
\newcommand{\DD}{\ensuremath{D}\xspace} % distribution D.
\newcommand{\dist}{\DD} % distribution D.
\newcommand{\Otilde}{\wt{O}}

\newcommand{\SamCom}{\mathsf{m}}
\newcommand{\metric}{{{\mathsf{d}}}}

\newcommand{\iid}{i.i.d.\xspace}

\usepackage[utf8]{inputenc}
\newcommand{\Levy}{Lévy\xspace}

\newcommand{\con}{{\bm \upalpha}}

\newcommand{\msr}{\dist}

\newcommand{\OnlTam}{\mathsf{On}\Tam}
\newcommand{\OnTam}{\OnlTam}
\newcommand{\OffTam}{\mathsf{Off}\Tam}

\newcommand{\mal}{\wt}
\newcommand{\X}{\cX} % X
\newcommand{\Y}{\cY} % X
\newcommand{\C}{\cC}
\newcommand{\D}{\cD}
\renewcommand{\H}{\cH}

\newcommand{\Risk}{\mathsf{Risk}}
\newcommand{\AdvRisk}{\mathsf{AdvRisk}}

\newcommand{\train}{\cS}

\newcommand{\Tam}{\mathsf{Tam}}

\newcommand{\uDist}{\mathbf{u}}

\newcommand{\uDistVec}{\ol{\uDist}}

\newcommand{\uVec}{\ol{u}}
\newcommand{\vVec}{\ol{v}}

\renewcommand{\th}{^\mathrm{th}}
\newcommand{\twist}[2]{\langle #1 \,\, \| \, {#2}\rangle}

\newcommand{\pfix}[2]{ {#1}_{\leq #2}}

\newcommand{\aSF}{\mathsf{a}}
\newcommand{\gSF}{\mathsf{g}}
\newcommand{\hSF}{\mathsf{h}}

\newcommand{\avr}[2]{\ifthenelse{\equal{#2}{}}{\aSF({#1})}{\ifthenelse{\equal{#2}{0}}{\aSF(\emptyset)}{\aSF({#1}_{\leq #2})}}}

\newcommand{\avrMax}[2]{\ifthenelse{\equal{#2}{}}{\aSF^*({#1})}{\ifthenelse{\equal{#2}{0}}{\aSF^*(\emptyset)}{\aSF^*({#1}_{\leq #2})}}}

\newcommand{\avrApp}[2]{\ifthenelse{\equal{#2}{}}{\tilde{\aSF}({#1})}{\ifthenelse{\equal{#2}{0}}{\tilde{\aSF}(\emptyset)}{\tilde{\aSF}({#1}_{\leq #2})}}}

\newcommand{\avrAppMax}[2]{\ifthenelse{\equal{#2}{}}{\tilde{\aSF}^*({#1})}{\ifthenelse{\equal{#2}{0}}{\tilde{\aSF}^*(\emptyset)}{\tilde{\aSF}^*({#1}_{\leq #2})}}}

\newcommand{\ArgMax}[2]{\ifthenelse{\equal{#2}{}}{\hSF({#1})}{\ifthenelse{\equal{#2}{0}}{\hSF(\emptyset)}{\hSF({#1}_{\leq #2})}}}

\newcommand{\AppArgMax}[2]{\ifthenelse{\equal{#2}{}}{\tilde{\hSF}({#1})}{\ifthenelse{\equal{#2}{0}}{\tilde{\hSF}(\emptyset)}{\tilde{\hSF}({#1}_{\leq #2})}}}

\newcommand{\gain}[2]{\ifthenelse{\equal{#2}{}}{\gSF(#1)}{\gSF(#1_{\leq #2})}}
\newcommand{\gainMax}[2]{\ifthenelse{\equal{#2}{}}{\gSF^*(#1)}{\gSF^*(#1_{\leq #2})}}
\newcommand{\gainApp}[2]{\ifthenelse{\equal{#2}{}}{\tilde{\gSF}(#1)}{\tilde{\gSF}(#1_{\leq #2})}}
\newcommand{\gainAppMax}[2]{\ifthenelse{\equal{#2}{}}{\tilde{\gSF}^*(#1)}{\tilde{\gSF}^*(#1_{\leq #2})}}

\newcommand{\pr}[2][]{\Pr_{\ifthenelse{\isempty{#1}}{}{{#1}}}\left[{#2}\right]}

\newcommand{\state}{\mathsf{st}}

\newcommand{\problem}{\ensuremath{\mathcal{P}}\xspace}

\newcommand{\e}{\mathrm{e}}

\renewcommand{\D}{\cD}

\usepackage{graphicx}

\newcommand{\remove}[1]{}

\newcommand{\ol}{\overline}
\newcommand{\wt}[1]{\widetilde{#1}}

\newcommand{\se}{\subseteq}

\newcommand{\angles}[1]{\langle #1 \rangle}

\newcommand{\set}[1]{\left\{ #1 \right\}}

\newcommand{\bits}{\{0,1\}}

\newcommand{\abf}[1]{\lVert #1 \rVert}
\newcommand{\norm}[1]{\abf{#1}}

\newcommand{\R}{{\mathbb R}}
\newcommand{\N}{{\mathbb N}}

\newcommand{\cA}{{\mathcal A}}
\newcommand{\cB}{{\mathcal B}}
\newcommand{\cC}{{\mathcal C}}
\newcommand{\cD}{{\mathcal D}}
\newcommand{\cE}{{\mathcal E}}

\newcommand{\cH}{{\mathcal H}}

\newcommand{\cR}{{\mathcal R}}
\newcommand{\cS}{{\mathcal S}}

\newcommand{\cX}{{\mathcal X}}
\newcommand{\cY}{{\mathcal Y}}

\newcommand{\bfc}{\mathbf{c}}

\usepackage{upgreek}
\usepackage{bm}

\newcommand{\eps}{\varepsilon}

\usepackage{bbm}
\newcommand{\indic}{\mathbbm{1}}

\newcommand{\poly}{\operatorname{poly}}
\newcommand{\polylog}{\operatorname{polylog}}

\newcommand{\Exp}{\operatorname*{\mathbb{E}}}
\newcommand{\Ex}{\Exp}

\newcommand{\Supp}{\operatorname{Supp}}

\newtheorem{theorem}{Theorem}[section]
\theoremstyle{plain}
\newtheorem{claim}[theorem]{Claim}

\newtheorem{lemma}[theorem]{Lemma}

\theoremstyle{definition}

\newtheorem{definition}[theorem]{Definition}

\theoremstyle{definition}
\newtheorem{remark}[theorem]{Remark}

\newcommand{\sdotfill}{\textcolor[rgb]{0.8,0.8,0.8}{\dotfill}} %change to \cdotfill later

\makeatletter
\def\th@protocol{%
    \normalfont % body font
    \setbeamercolor{block title example}{bg=orange,fg=white}
    \setbeamercolor{block body example}{bg=orange!20,fg=black}
    \def\inserttheoremblockenv{exampleblock}
  }
\makeatother
\theoremstyle{protocol}

\newtheorem{proto}[theorem]{Protocol}
\newtheorem{protoc}[theorem]{Protocol}

\newcommand{\namedref}[2]{#1~\ref{#2}}

\newcommand{\torestate}[3]{%
\expandafter \def \csname BBRESTATE #2 \endcsname{#3}
\theoremstyle{plain}
\newtheorem{BBRESTATETHMNUM#2}[theorem]{#1}
\begin{BBRESTATETHMNUM#2}\label{#2}\csname BBRESTATE #2 \endcsname   \end{BBRESTATETHMNUM#2}
\newtheorem*{BBRESTATETHMNONNUM#2}{\namedref{#1}{#2}}
}

\newcommand{\restate}[1]{\begin{BBRESTATETHMNONNUM#1}[Restated] \csname BBRESTATE #1 \endcsname
\end{BBRESTATETHMNONNUM#1}}

%% file: intro.tex
\section{Introduction} \label{sec:intro}
Learning predictors is the task of outputting a hypothesis $h$ using a training set $\cS$ in such a way that $h$ can predict the correct label $c(x)$ of unseen instances such as $x$ with high probability. A normal successful learner, however, could be vulnerable to adversarial perturbations. In particular, it was shown \citep{Szegedy:intriguing,Evasion:TestTime,Adversarial::Harnessing} that deep neural nets (DNNs) are vulnerable to so called adversarial examples %$\mal{x}$ 
that are the result of small (even imperceptible to human eyes) perturbations on the original input $x$. Since the introduction of such attacks, many works have studied defenses against them and more attacks are introduced afterwards \citep{Evasion:TestTime,biggio2014security,Adversarial::Harnessing,Defenses:Distillation,CarliniWagner,Adversarial::FeatureSqueezing,madry2017towards}. 

A fundamental question in robust learning is  whether one can design learning algorithms that achieve ``generalization'' even under such adversarial perturbations. Namely, we want to know  when we  can  learn a robust classifier $h$ that still correctly classifies its inputs even if they are adversarially perturbed in a limited way. Indeed, one can ask when the $(\eps,\delta)$ PAC (probably approximately correct) learning \citep{Valiant:PAC} is possible in adversarial settings. More formally, the goal here is to learn a robust $h$ from the data set $\cS$ consisting of $m$ independently sampled labeled (non-adversarial) instances in such a way that, with probability $1-\delta$ over the learning process, the produced $h$ has error at most $\eps$ even under ``limited'' adversarial perturbations of the input. This limitation is carefully defined by some metric $\metric$ defined over the input space $\X$ and some upper bound ``budget'' $b$  on the amount of perturbations that the adversary can introduce. I.e., we would like to minimize
\begin{equation*}
\AdvRisk(h)=\Pr_{x \gets D}[\exists~\mal{x} \colon d(x,\mal{x})\leq b, h(\mal{x}) \neq c(\textcolor{red}{\mal{x}})]\leq \eps    
\end{equation*}
where $\AdvRisk$ is the ``adversarial'' risk, and $c(\cdot)$ is the ground truth (i.e., the concept function).  

\paragraph{Error-region adversarial risk.} The above notion of adversarial risk  has   been used implicitly or explicitly in previous work \citep{gilmer2018adversarial,diochnos2018adversarial,bubeck2018adversarial2,degwekar2019computational,ford2019adversarial} and was  formalized by \citet{diochnos2018adversarial} as the ``error-region'' adversarial risk, because adversary's goal here is to push $\mal{x}$ into the error region 
$$\cE= \set{x \mid h(x) \neq c(x)}.$$ 

\paragraph{Corrupted-input adversarial risk.} Another notion of adversarial risk (that is similar, but still different from the error-region adversarial risk explained above) has been  used  in many works such as  \citep{feige2015learning,madry2017towards} in which the perturbed $\mal{x}$ is interpreted as a ``corrupted input''. Namely, here the goal of the learner is to find the label of the original \emph{untampered} point $x$ by only having its corrupted version $\mal{x}$, and thus adversary's success criterion is to reach $d(x,\mal{x})\leq b, h(\mal{x}) \neq c(\textcolor{red}{x})$. Hence, in that setting, the goal of the learner is to find an $h$ that minimizes
\begin{equation*}
    \Pr_{x \gets D}[\exists~\mal{x} \colon d(x,\mal{x})\leq b, h(\mal{x}) \neq c(\textcolor{red}{x})].
\end{equation*}
It is easy to see that, if the ground truth $c(x)$ does not change under $b$-perturbations, $c(x) =c(\mal{x})$, the two notions of error-region and corrupted-input adversarial risk will be equal.
In particular, this is the case for practical distributions of interest, such as images or voice, where sufficiently-small perturbations usually do not change human's judgment about the true label. %Therefore, both definitions can be used for measuring adversarial risk in these contexts. 
However,  if $b$-perturbations can change the ground truth, $c(x) \neq c(\mal{x})$, the two definitions are incomparable. 
%While the error-region definition requires the misclassification of $\mal{x}$, the corrupted input definition does not do so. 
%(e.g., it could happen that the corrupted input $\mal{x}$ is still correctly classified $h(\mal{x})=c(\mal{x})$, while the attacker is succeeding if $h(\mal{x})\neq c(x)$.

Several works have already studied  PAC learning with provable guarantees under adversarial perturbations \citep{bubeck2018adversarial,cullina2018pac,feige2018robust,attias2018improved,khim2018adversarial,yin2018rademacher,montasser2019vc}. However, all these works use the   \emph{corrupted-input} notion of adversarial risk. In particular, it is proved by \citet{attias2018improved} that robust learning might require more data, but it was also shown by  \citet{attias2018improved,bubeck2018adversarial} that in natural settings, if robust classification is feasible, robust classifiers could be found with a sample complexity  that is  only \emph{polynomially} larger than that of normal learning.
This leads us to the our central question:

\begin{quote}
    
\emph{What problems are PAC learnable under   evasion attacks that perturb instances into the error region? If PAC learnable, what is their sample complexity?}

\end{quote}

%Learning under adversarial perturbations  during the training or testing phase has been the subject of intense studies in recent years. Training time attacks, also known as data poisoning or contamination attacks \citep{biggio2012poisoning,papernot2016towards}, allow an adversary to perturb $\cS$ into a ``close'' data set $\mal{\cS}$ with the hope of degrading the quality of $h$. Test time attacks, known as evasion attacks finding adversarial examples \citep{Szegedy:intriguing,Evasion:TestTime,Adversarial::Harnessing}, allow an adversary  to perturb the original test instance $x$ into a ``close'' perturbed $\mal{x}$ with the hope of making $h$ make a mistake. In both scenarios, an adversarially robust learner is a learner that simply achieves its goal even under such attacks.

%Since the introduction of poising and evasion attacks, a so called arms race has began with defenses against such attacks, more attacks against new defences, etc., in both contexts of  and evasion attacks finding adversarial examples (e.g., see .

\subsection{Our Contribution}

In this work, we initiate a formal study of PAC learning under adversarial perturbations, where the goal of the adversary is to increase the error-region adversarial risk using small (sublinear $o(\norm{x})$) perturbations of the inputs $x$.  Therefore, in what follows, whenever we refer to adversarial risk, by default it means the error-region variant. 

\paragraph{Result 1: exponential lower bound on sample complexity.}
Suppose the instances of a learning problem come from a metric probability space $(\X,D,\metric)$ where $D$ is a distribution and $\metric$ is a metric defining some norm $\norm{\cdot}$. Suppose  the input instances have norms $\norm{x} \approx n$ where $n$ is a parameter related (or in fact equal) to the data dimension. One natural setting of study for PAC learning is to study attackers that can only perturb $x$ by a \emph{sublinear} amount $o(\norm{x}) = o(n)$ (e.g., $\sqrt{n}$).

Our first result is to prove a strong lower bound for the sample complexity of PAC learning in this setting. We prove that for many theoretically natural input spaces of high dimension $n$ (e.g., isotropic Gaussian in dimension $n$ under  $\ell_2$ perturbations),  PAC learning of certain problems under sublinear perturbations of the test instances requires   \emph{exponentially} many samples in $n$, even though  the problem in the no-attack setting is PAC learnable using polynomially many samples.  This holds e.g., when we want to learn half spaces in dimension $n$ under such distributions (which is possible in the no-attack setting). We note  that even though PAC learning is defined for all distributions,  proving such lower  bound for a specific input distribution $D$ over $\X$ only makes the negative result \emph{stronger}.   Our lower bound is in contrast with previously proved results \citep{attias2018improved,bubeck2018adversarial,montasser2019vc,cullina2018pac} in which the gap between the sample complexity of the normal and robust learning is only \emph{polynomial}. However, as mentioned before, all these previous results are proved  using the \emph{corrupted-input} variant of adversarial risk.

Our result extends to any learning problem where input space $\X$, the metric $\metric$ and the distribution  $D$ defined over them, and the class of concept functions $\C$  have the following two conditions. 
\begin{enumerate}
    \item The inputs $\X$ under the distribution $D$ and  small perturbations measured by the  metric $\metric$  forms a \emph{concentrated} metric probability space \citep{ledoux2001concentration,milman1986asymptotic}. A concentrated space has the property that relatively small events (e.g., of measure $0.1$) under small (e.g., smaller than the diameter of the space) perturbations expand  to cover almost all measure $\approx 1$ of the input space.
    \item The set of concept functions $\C$ are complex enough to allow proving   lower bounds for the sample complexity for (distribution-dependent) PAC learners in the \emph{no-attack} setting under the \emph{same distribution} $D$. Distribution-dependent sample complexity lower bounds are known for certain settings \citep{long1995sample,balcan2013active,sabato2013distribution}, however, we use a more relaxed condition that can be applied to broader settings. In particular, we require that for a sufficiently small $\eps$, there are two concept functions $c_1,c_2$ that are equal for  $1-\eps$ fraction of inputs sampled from    $D$ (see Definition \ref{def:close}).
\end{enumerate}

Having the above two  conditions,  our proof proceeds as follows \textbf{(I)} We show that  the  (normal) risk $\Risk(h)$ of a hypothesis produced by \emph{any} learning algorithm  with sub-exponential sample complexity cannot be as large as an inverse polynomial over the dimension. \textbf{(II)} We then use ideas from the works (e.g., see \citep{mahloujifar2018curse}) to show that such sufficiently large risk will expand into a large \emph{adversarial} risk of almost all  inputs, due to the measure concentration  the input space. 

\begin{remark}[Approximation error in error-region robust learning] If a learning problem is \emph{realizable} in the no-attack setting, i.e., there is a hypothesis $h$ that has  risk zero over the test instances, it means that the same hypothesis $h$ will have adversarial  (true) risk zero over the test instances as well, because any perturbed point  is still going to be correctly classified.   This is in contrast with corrupted-input notion of adversarial risk that even in realizable problems, the smallest corrupted-input (true) adversarial risk  could still be large, and even at odds with correctness \citep{tsipras2018robustness}. This means that our results rule out (efficient) PAC learning even in the \emph{agnostic} setting as well, because in the realizable setting there is at least one hypothesis with error-region adversarial risk zero while (as we prove), in some settings learning a model with  adversarial risk  (under sublinear perturbations) close to zero requires exponentially many samples.
\end{remark}

\paragraph{Result 2: ruling out PAC learning under hybrid attacks.}
We then study PAC learning under adversarial perturbations that happen during \emph{both} training and testing phases. We formalize \emph{hybrid} attacks in which the final evasion attack is preceded by a  poisoning attack \citep{biggio2012poisoning,papernot2016towards}. This attack model bears similarities to ``trapdoor attacks'' \citep{gu2017badnets} in which a poisoning phase is involved before the evasion attack, and here we give a formal definition for PAC learning under such attacks. Our definition of hybrid attacks is general and can incorporate any notion of adversarial risk, but our results for hybrid attacks use the \emph{error-region} adversarial risk.

Under hybrid attacks, we show that PAC learning is sometimes \emph{impossible} all together, even though it is possible without such attacks. For example, even if the VC dimension of the concept class is bounded by $n$, if the adversary is allowed to poison only $1/n^{10}$ fraction of the $m$ training examples, then it can do so in such a way that a subsequent evasion attack could then  increase the adversarial risk to $\approx 1$. This means that PAC learning is in fact impossible under such hybrid attacks. 

We also note that classical results about malicious noise \citep{Valiant::DisjunctionsConjunctions,KearnsLi::Malicious} 
and nasty noise \citep{NastyNoise}
could be interpreted as ruling out PAC learning under poisoning attacks. However, there are two differences: \textbf{(I)} The adversary in these previous works needs to change a \emph{constant} fraction of the training examples, while our attacker changes only an \emph{arbitrarily small} inverse polynomial fraction of them.
 \textbf{(II)} Our poisoning attacker only \emph{removes} a fraction of the training set, and hence it does \emph{not} add any misclassified examples to the pool. Thus the poisoning attack used here is a clean/correct label attack \citep{Mahloujifar2018:ALT,shafahi2018poison}.

%poisoning (e.g., see \citep{awasthi2014powerjournal,xiao2015feature,papernot2016towards,rubinstein2009antidote,diakonikolas2017statistical,Mahloujifar2018:ALT,diakonikolas2018sever,prasad2018robust})

%% file: NewDefs.tex
\section{Defining Adversarially Robust PAC Learning} \label{sec:defs}

\paragraph{Notation.} By $\Otilde(f(n))$ we refer to the set of all functions of the form $O(f(n)   \log(f(n))^{O(1)})$. We  use capital calligraphic letters (e.g., $\D$) for sets and capital non-calligraphic letters (e.g., $D$) for distributions.  $x \gets D$ denotes sampling $x$  from $D$. For an event $\cS$, we let $D(\cS)=\Pr_{x \gets D}[x \in \cS]$.

A classification problem $\problem=(\XX,\YY,\C,\D,\H)$ is specified by the following components. 
The set $\XX$ is the set  of possible \emph{instances}, %and  
\YY is the set of possible \emph{labels}, %and
%by default we use $\YY = \bits$,
$\D$ is a class of distributions  over instances $\XX$. In the standard setting of PAC learning, $\D$ includes all distributions, but since we deal with \emph{negative} results, we sometimes work with fixed $\D=\set{D}$  distributions, and show that even \emph{distribution-dependent} robust PAC learning is sometimes hard. In that case, we represent the problem as $\problem=(\XX,\YY,\C,\textcolor{red}{D},\H)$.
The set $\C \subseteq \YY^\XX$ is the \emph{concept class} and  $\H \subseteq \YY^\XX$ is
 the  \emph{hypothesis class}. In general, we can allow \emph{randomized} concept and hypothesis functions to model, in order, label uncertainly (usually modeled by a joint distribution over instances and labels) and randomized predictions. All of our results extend to randomized learners and randomized hypothesis functions, but for simplicity of presentation, we treat them as deterministic mappings. 
 %by default we work with deterministic concept and hypothesis functions. 
By default, we  consider 0-1 \emph{loss functions} where  $\loss(y',y)=\indic [y'=y]$.
%if $y'=y$ and $\loss(h,x,y)=1$ if $h(x) \neq y$.
For a given distribution $D \in \D$ and a concept function $c \in \C$,
the \emph{risk} of a hypothesis $h \in \H$ is the expected loss of $h$ with respect to $D$, namely $\Risk(D,c,h) = \Pr_{x \gets D}[\loss(h(x),c(x))]$. 
An \emph{example} $z$ is a pair $z=(x,y)$ where $x \in \XX$ and $y \in \YY$. An example is usually sampled by first sampling $x\gets D$ for some $D \in \D$ followed by letting $y=c(x)$ for some $c \in \C$. 
A \emph{sample} sequence $\train=(z_1,\dots,z_m)$ is a sequence of $m$ examples. As is usual, sometimes we might refer to a sample sequence as  the  training \emph{set}. By  $\train \gets (D,c(D))^m$  we denote the process of obtaining $\train$ by sampling $m$  iid samples from $D$ and labeling them by $c$.
%A hypothesis $\hconcept$ is \emph{consistent} with a sample set (or sequence) $\train$  if and only if $\hconcept(x) = y$ for all $(x, y)\in \train$.  We assume that instances, labels, and hypotheses  are encoded as strings over some %finite or infinite  alphabet such that given a hypothesis $h$ and an instance $x$,  $h(x)$ is computable in polynomial time.

Our learning problems $\problem_n=(\XX_n,\YY_n,\C_n,\D_n,\H_n)$ are usually parameterized by $n$ where $n$ denotes the ``data dimension'' or (closely) capture the  bit length of the instances. Thus, the ``efficiency'' of the algorithms could depend on $n$. Even in this case, for simplicity of notation, we might simply write $\problem=(\XX,\YY,\C,\D,\H)$. By default, we will have $\C \se \H$, in which case we call $\problem$ \emph{realizable}. This means that for any  training set for $c\in \C,D\in \D$, there is  a hypothesis that has empirical and true risk zero; though finding such $h$ might be challenging.
%As we prove negative results in this work, working with the special case of realize setting only makes our results stronger and apply to arbitrary hypotheses classes.

\paragraph{Evasion attacks.} An evasion attacker  $\Adv$ is one that changes the test instance $x$, denoted as $\mal{x} \gets \Adv(x)$. The behavior and actions taken by $\Adv$ could, in general, depend on the choices of $D\in \D,c\in\C$, and $h\in \H$. As a result, in our notation, we provide $\adv$ with access to $D,c,h$ by giving them as special inputs to $\Adv$,\footnote{This dependence is information theoretic, and for example, $\Adv$ might want to  find $\mal{x}$ that is misclassified, in which case its success is defined as $h(\mal{x}) \neq c(\mal{x})$ which depends on both $h,c$. %Or, an attacker might want to use only correct/clean labels, in which case its actions depends on the concept function $c$.
} 
denoting the process as $\mal{x} \gets \Adv[D,c,h](x)$. We  use calligraphic font $\advC$ to denote a  \emph{class/set} of attacks. For example, $\advC$ could contain all attackers who could  change test instance $x$ by at most $b$ perturbations under a metric defined over $\X$.

\paragraph{Poisoning attacks.}
A poisoning attacker  $\Adv$ is one that changes the training sequence as $\mal{\train} \gets \Adv(\train)$. Such attacks, in general, might add examples to $\cS$, remove examples from $\cS$, or do both. The behavior and actions taken by $\Adv$ could, in general, depend on the choices of $D\in \D,c\in\C$ (but not on $h\in \H$, as it is not produced by the learner at the time of the poisoning attack)\footnote{For example, an attack model might require $\Adv$ to choose its perturbed instances still using \emph{correct/clean} labels, in which case the attack is restricted based on the choice of $c$).}. As a result, we provide implicit access to $D,c$ by giving them as special inputs to $\Adv$, denoting the process as $\mal{\train} \gets \Adv[D,c](\train)$. We use calligraphic font  $\advC$ to denote a  \emph{class/set} of attacks. For example, $\advC$ could contain attacks that change  $1/n$ fraction of $\cS$ only using  clean labels \citep{mahloujifar2018curse,shafahi2018poison}.

\paragraph{Hybrid attacks.} A hybrid attack  $\Adv=(\adv_1,\adv_2)$ is a two phase attack in which $\adv_1$ is a poisoning attacker and $\adv_2$ is an evasion attacker. One subtle point is that $\adv_2$ is also aware of the  internal state of $\adv_1$, as they are a pair of coordinating attacks. More formally, $\adv_1$ outputs an extra ``state'' information $\state$   which will be given as an extra input to $\adv_2$. As discussed above, $\adv_1$ can depend on $D,c$, and $\adv_2$ can depend on  $D,c,h$ as defined for  evasion and poisoning attacks.  

We now define PAC learning under adversarial perturbation attacks. To do so, we need to first define our notion of adversarial risk. We will do so by employing the \emph{error-region} notion adversarial risk as formalized in~\cite{diochnos2018adversarial}  adversary aims to misclassify the perturbed instance $\mal{x}$.

\begin{definition}[Error-region (adversarial)  risk] Suppose $\adv$ is an evasion adversary and let $D,c,h$ be fixed. The \emph{error-region} (adversarial) risk  is defined as follows.
\begin{equation*}
    \AdvRisk_{{\adv}}(D, c,h) = \Pr_{x \gets D, \mal{x} \gets \adv[D,c,h](x)}[h(\mal{x})\neq c(\mal{x})].
\end{equation*} 
For randomized $h$, the above probability is also over the randomness of $h$  chosen after $\mal{x}$ is selected.
\end{definition}

We  now define PAC learning under hybrid attacks, from which one can derive also the definition of PAC learning under   evasion attacks and under  poisoning attacks.

\begin{definition}[PAC learning under   hybrid attacks] \label{def:hybrid}

Suppose $\problem_n=(\XX_n,\YY_n,\C_n,\D_n,\H_n)$ is a realizable classification problem, and suppose $\advC$  is a class of hybrid attacks for $\problem_n$. $\problem_n$ is PAC learnable with sample complexity $\SamCom(\eps,\delta,n)$ under hybrid attacks of $\advC$ , if there is a  learning algorithm $L$ such that for every $n$, $0<\eps,\delta<1,c \in \C,D\in\D,$ and $(\adv_1,\adv_2)\in\cA$, if $m=\SamCom(\eps,\delta,n)$, then
\begin{equation*}
    \Pr_{\substack{\train \gets (D,c(D))^m, \\ {(\mal{\train},\state) \gets \adv_1[D,c](\train)}, \\ h \gets L(\mal{\train})}}\left[\AdvRisk_{{\adv_2[D, c,h,\state]}}(h,c,D) > \eps \right] \leq \delta .
\end{equation*}
PAC learning under (pure) poisoning attacks or evasion attacks could be derived from Definition~\ref{def:hybrid} by letting either of $\adv_1$ or $\adv_2$ be a trivial attack that does no tampering at all.

%We say $L$ has \emph{polynomial sample complexity} if $\SamCom(\eps,\delta,n) = \poly(n/\eps \delta)$, and is \emph{efficient} if it runs in time $\poly(n/\eps \delta)$.  
\end{definition}

We also note  that one can obtain other definitions of PAC learning under evasion  or hybrid attacks in Definition~\ref{def:hybrid}   by using  other forms of adversarial risk, e.g., corrupted-input adversarial risk~\citep{feige2015learning,feige2018robust,madry2017towards,schmidt2018adversarially,attias2018improved}

%% file: results.tex
\section{Lower Bounds for PAC Learning under Evasion and Hybrid Attacks}
Before proving our main results, we need to recall the notion of Normal \Levy families, and define a desired and common property of set of concept functions with respect to the distribution of inputs.

%\subsection{Definitions for Metric Probability Spaces}

\paragraph{Notation.} Let $(\X,\metric)$ be a metric space. 
%We let
%$\diam^\metric(\X) = \sup\set{\metric(x,y) \mid x,y \in \X_i}$ to denote the diameter of $\X$ under $\metric$, and we use 
%$\Ball_b (x)=\set{x' \mid \metric(x,x') \leq b}$ be   the ball of radius $b$ centered at $x$. 
%When $\metric$ is clear from the context, we simply write %$\diam (\X)$ and $\Ball_b (x)$. 
For   $\cS\se \X$, by $\metric(x,\cS) = \inf\set{\metric(x,y) \mid  y \in \cS}$  we denote the distance of a point $x$ from $\cS$. We also let $\cS_b= \set{y \mid \metric(x,y) \leq b, x \in \cS}$ be the \emph{$b$-expansion} of $\cS$. When there is also a measure $D$ defined over the metric space $(\X,\metric)$, the \emph{concentration function} is  defined and denoted as     $ \con(b)= 1-\inf \set{\Pr_D[\cE_b] \mid \Pr_D[\cE] \geq 1/2}.$

\begin{definition}[Normal \Levy families] \label{def:Levy}
A  \emph{family} of  metric probability spaces $(\X_n, \metric_n, \msr_n)_{i\in\N}$ with  concentration function $\con_n(\cdot)$ is  called a \emph{normal \Levy family}  if there are $k_1,k_2$, such that\footnote{Another common formulation of Normal \Levy families uses $\con_n(b) \leq k_1 \cdot \e^{-k_2\cdot b^2 \cdot n}$, but here we scale the distances up by $n$ to achieve ``typical norms'' to be $\approx n$, which is the dimension.} \begin{equation*}
    \con_n(b) \leq k_1 \cdot \e^{-k_2\cdot b^2 / n}
\end{equation*}
\end{definition}

\paragraph{Examples.} Many natural metric probability spaces are Normal \Levy families. For example, all the following examples under normalized distance (to make the typical norms $\approx n$) are normal \Levy families as stated in Definition~\ref{def:Levy}: the unit $n$-sphere with uniform distribution under the Euclidean or geodesic distance, $\R^n$ under Gaussian distribution and  Euclidean distance, $\R^n$ under Gaussian distribution and Euclidean distance, the unit $n$-cube and unit $n$-ball under the uniform distribution and Euclidean distance,  any product distribution of dimension $n$ under the Hamming distance. See \citep{ledoux2001concentration,giannopoulos2001euclidean,milman1986asymptotic} for more examples.

%\subsection{Properties of Normal \Levy Families}

%The following lemma directly follows from the definition of Normal \Levy families. (See Par2 of Theorem 3.7 in~\cite{mahloujifar2018curse}.)
%\begin{lemma} \label{lem:expansion}
%Let $(\X_n, \metric_n, \msr_n)_{n\in\N}$ be a Normal \Levy family, and let $\cS$ be measurable with probability bigger than $\alpha$. If $$b\geq  {\sqrt{ n/k_2} \cdot \big(\sqrt{\ln({k_1}/{\alpha})} + \sqrt{\ln(k_1/\beta})\big)},$$  then $D_n(\cS_b) \geq 1-\beta$. 
 %\Snote{I changed $S_b$ to $D_n(S_b)$, is that ok?} \Mnote{sure!}
%\end{lemma}

The following lemma was proved in~\cite{mahloujifar2018curse} when  Normal \Levy input spaces.
%\Snote{Since we are borrowing this lemma, can we remove the notation about concentration function together with Lemma~\ref{lem:expansion}?} \Mnote{sounds good; actually even this can go to appendix.}

\begin{lemma}  \label{lem:advrisk}
Let the input space of a hypothesis classifier $h$ be    a Normal \Levy family $(\X_n, \metric_n, \msr_n)_{i\in\N}$. If the risk of $h$ with respect to the ground truth concept function $c$ is bigger than $\alpha$, $\Risk(D_n,c,h) \geq \alpha$, and if an adversary $\adv$ can perturb instances by up to $b$ in metric $\metric_n$ for
$$b=  {\sqrt{ n/k_2} \cdot \big(\sqrt{\ln({k_1}/{\alpha})} + \sqrt{\ln(k_1/\beta})}\big),$$
 then the adversarial risk is $\AdvRisk_\adv(D,h,c)  \geq 1-\beta$.
\end{lemma}

\begin{definition}[$\alpha$-close function families] \label{def:close} Suppose $D$ is a distribution over $\X$, and let $\C$ be a set of functions  from $\X$ to some set $\Y$. We call $\C$ \emph{$\alpha$-close} with respect to $D$, if there are   $c_1,c_2 \in \C$ such that $\Pr_{x \gets D}[c_1(x) \neq c_2(x)]=\alpha$.
%We call $\C$ \emph{uniformly $\alpha$-close} with respect to $D$, if there is a joint distribution $(\bfc_1,\bfc_2)$ where both coordinates are \emph{uniformly} distributed over $\C$, and that for all $(c_1,c_2) \gets (\bfc_1,\bfc_2)$, it both holds that   $c_1 , c_2 \in \C$ and  that   $\Pr_{x \gets D}[c_1(x) \neq c_2(x)]=\alpha$.
\end{definition}

\paragraph{Examples.} The set of homogeneous half spaces in $\R^n$ are  $\alpha$-close for all $\alpha \in (0,1]$ under any of the following natural distributions: uniform over the unit sphere, uniform inside the unit ball, and isotropic Gaussian. This can be proved by picking two half spaces that their disagreement region under the mentioned distributions is exactly $\alpha$. The set of (monotone, or not necessarily monotone) conjunctions are $\alpha$-close for  $\alpha=2^{-k}$ for all $k\in\set{2,\dots,n}$ under the uniform distribution over $\bits^n$. This can be proved by looking at $c_1 = x_1\wedge \ldots \wedge x_{k-1}$ and $c_2 = x_1\wedge \ldots \wedge x_{k-1} \wedge x_k = c_1 \wedge x_{k}$. 
Since all the variables that appear in $c_1$ also appear in $c_2$, we have that $\Pr_{x\gets\bits^n}[c_1(x) \neq c_2(x)]$ is equal to
%\begin{displaymath}
$\Pr_{x\gets\bits^n}[(c_1(x) = 1) \wedge (c_2(x) = 0)]$,
%\end{displaymath}
and as a consequence this is equal to
\remove{
\begin{displaymath}
\Pr_{x\gets\bits^n}\left[ \bigwedge_{i=1}^{k-1} (x_i = 1)\right] - \Pr_{x\gets\bits^n}\left[\left(\bigwedge_{i=1}^{k-1} (x_i = 1)\right)\wedge(x_{k} = 0)\right]
\end{displaymath}
which is }
%\begin{displaymath}
$2^{-(k-1)} - 2^{-k} = 2^{-k}$.
%\end{displaymath}

We now state and prove our main results. Theorem \ref{thm:main} is stated in the \emph{asymptotic} form considering attack families that attack the problem for sufficiently large index $n \in \N$ of the problem. We describe a quantitative variant afterwards (Lemma \ref{lem:main}).
\begin{theorem}[Limits of adversarially robust PAC learning] \label{thm:main}
Suppose $\problem_n=(\XX,\YY,\C,\D,\H)$ is a realizable classification problem and that $\XX$ is a Normal \Levy Family (Definition~\ref{def:Levy}) over $D$ and a metric $\metric$, and that $\C$ is $\Theta(\alpha)$-close with respect to $D$ for all $\alpha \in [2^{-\Theta(n)},1]$.  
%Let $\advC_b$ be the class of evasion adversaries that can perturb an input instance $x$ by at most $b=b(n)$ under the metric $\metric$. 
Then, the following hold even for  PAC learning with  parameters $\eps=0.9,\delta=0.49$.
\begin{enumerate}
    \item {\sf Sample complexity of PAC learning robust fo evasion attacks:} \label{part:evasion}
    \begin{enumerate}
    \item {\bf  Exponential lower bound:} \label{part:exp}
    %If  $b=\rho \cdot n$, for $\rho=\rho(n) <1$, then   any PAC learning algorithm  for $\problem_n$ under the attacks of $\advC$, even for  parameters $\eps,\delta=0.9$ requires sample complexity at least $m \geq \Omega(2^{\Omega(\rho ^2   n)})$. 
   Any PAC learning algorithm  that is robust against  \emph{all} attacks with a sublinear tampering   $b =o(n)$ budget under the metric $\metric$ requires exponential sample complexity $m \geq 2^{\Omega(n)}$.
    
            \item {\bf Super-polynomial  lower bound:} \label{part:superpoly}  PAC learning that is robust against against all tampering attacks with budget $b=\Otilde(\sqrt{n})$,  requires  at least $m \geq n^{\omega(1)}$ many samples.
    \end{enumerate} 
    
    \item {\sf Ruling out PAC learning robust to hybrid attacks:} \label{part:ruleout} 
    
    Suppose the tampering budget of the evasion adversary can be any $b=\Otilde(\sqrt{n})$, and let $\cB_\lambda$ be any class of poisoning attacks that can remove $\lambda=\lambda(n)$ fraction of the training examples for an (arbitrary small) inverse polynomial $\lambda(n) \geq 1/\poly(n)$. Let $\cR$ be the class of hybrid attacks that first do  a poisoning by some $\mathsf{B} \in \cB_\lambda$ and then an evasion by some  adversary of budget $b=\Otilde(\sqrt{n})$. Then, $\problem_n$ is \emph{not} PAC learnable (regardless of sample complexity) under  hybrid attacks in $\cR$.
\end{enumerate}
\end{theorem}

%See  Section \ref{sec:proofs} for the full proof of Theorem \ref{thm:main}  using the ideas explained in %Section \ref{sec:intro}. In fact, as we will see in Section \ref{sec:proofs}, 
As we will see, Part \ref{part:exp} and Part \ref{part:superpoly} of Theorem \ref{thm:main} are special cases of the following more quantitative lower bound that might be of independent interest.
%We first prove a claim, using which we prove Theorem~\ref{thm:main}.
 
\begin{lemma} \label{lem:main} For the setting of Theorem \ref{thm:main}, if the tampering budget is  $b=\rho \cdot n$, for a fixed function $\rho=\rho(n) =o(1)$, then   any PAC learning algorithm  for $\problem_n$  under evasion attacks of tampering budget $b=b(n)$, even for  parameters $\eps=0.9,\delta=0.49$ requires sample complexity at least
\begin{equation*}
    m(n) \geq 2^{\Omega(\rho ^2  \cdot n)}.
\end{equation*}
\end{lemma}

\paragraph{Examples.} Here we list some natural scenarios that fall into the conditions of Theorem~\ref{thm:main}. All examples of Normal \Levy families listed after Definition~\ref{def:Levy} together with the concept class of half spaces satisfy the conditions of Theorem~\ref{thm:main} and hence cannot be PAC learned using a $\poly(n)$ number of samples. The reason is that one can always find two half spaces whose symmetric difference has measure exactly $\eps$. 
%In fact, even if we restrict ourselves to ``finite'' $O(n)$-bit precision representation of half spaces, one can still find $\Theta(\eps)$-far half spaces under the natural distributions mentioned for Normal \Levy families for $\eps$ as small as $\exp(-n)$.
Moreover, as discussed in examples following Definition~\ref{def:close}, even discrete problems such as learning monotone-conjunctions under the uniform distribution (and Hamming distance as perturbation metric) fall into the conditions of Theorem~\ref{thm:main}, for which a lower bound on their sample complexity (or even impossibility) of robust PAC learning could be obtained.

\begin{remark}[Evasion-robust PAC learning in the RAM computing model  with real numbers] 
We remark that if we allow (truly) real numbers represent the concept and hypothesis classes, one can even \emph{rule out} PAC learning (not just lower bounds on sample complexity) under similar perturbations describe in Part \ref{part:evasion}. 
Indeed, by inspecting the same proof of Theorem \ref{thm:main} for  Part \ref{part:evasion} one can get such results, e.g., for learning half-spaces in dimension $n$ when inputs come from isotropic Gaussian. However, we emphasize that such (seemingly) stronger lower bounds are not realistic, as in real settings, we eventually work with \emph{finite} precision to represent the concept functions (of half spaces). This makes the set of concept functions \emph{finite}, in which case the test error eventually reaches \emph{zero}, using perhaps exponentially many samples. Theorem \ref{thm:main}, however, has the useful feature that it applies even in those settings, as long as the concept functions are rich enough to allow the sufficiently close (but not too close) pairs under the distribution $D$ according to Definition \ref{def:close}.
\end{remark}

In what follows, we will first prove Lemma \ref{lem:main}. We will then use Lemma \ref{lem:main} to prove Theorem~\ref{thm:main}.

\begin{proof}[Proof of Lemma \ref{lem:main}] Let $m=\SamCom(0.9,0.49,n)$ be the sample complexity of the (presumed) learner $L$ that achieves $(\eps,\delta)$-PAC learning for $\eps=0.9,\delta=0.49$. If $m=2^{\Omega(n)}$ already, we are done, as it is even larger than what Lemma~\ref{lem:main} states, so let $m=2^{o(n)}$, and we will derive a contradiction.
Since the distribution $D$ is fixed, in the discussion below, we simply denote $\Risk(D,h,c)$ as $\Risk(h,c)$.

Recall that, by assumption, for all  $\eps\in[2^{-\Theta(n)},1]$, there are $c_1,c_2\in\C$  that are $\Theta(\eps)$-close under the distribution $D$. Because $m= 2^{o(n)}$, it holds that $1/m \geq \omega( 2^{-\Theta(n)})$, and so there are $c_1,c_2\in\C$ such that for $\Delta(c_1,c_2)  = \set{x\in \XX \mid c_1(x) \neq c_2(x)}$ we have
\begin{equation*}\Omega\left(\frac{1}{m}\right) \leq \Pr_{x \gets D}[x \in \Delta(c_1,c_2)] \leq \frac{1}{100m}.\end{equation*}
Now, consider  $m$ \iid samples that are given to the learner $L$ as a training set $\cS$. With probability at least $0.99$ of the sampling of $\cS$, all $x \in \cS$ would be outside $\Delta(c_1,c_2) $, in which case $L$ would have no way to distinguish $c_1$ from $c_2$. So, if we pick $c \gets \set{c_1,c_2}$ at random and pick test instance $x \gets (D\mid \Delta(c_1,c_2))$, the hypothesis $h = L(\cS)$ fails with probability at least $0.99/2$. Thus, we can fix the choice of $c \in  \set{c_1,c_2}$, such that with probability $0.99/2>0.49$  we  get a $h \gets L(\cS)$ where
\begin{equation*}\Risk(h,c)=\Pr_{x \gets D}[h(x) \neq c(x)] \geq \frac{1}{2} \cdot \Pr_{x \gets D}[x \in \Delta(c_1,c_2)]   \geq \Omega\left(\frac{1}{m}\right).\end{equation*}

For this fixed $c$ and any such learned hypothesis $h$ with $\Risk(h,c) =\Omega(1)/m$, by Lemma~\ref{lem:advrisk}, the adversarial risk reaches
$\AdvRisk_{\cA_b}(h,c)\geq 0.99 $ by an attack $\adv \in \cA_b$ that has tampering budget:
\begin{equation*}b = O(\sqrt{n}) \cdot \big(\sqrt {\ln(O(m))} + \sqrt{O(1)}\big) \leq  t \cdot (\sqrt{n \cdot \ln m})\end{equation*}
for universal constant $t$.
But, we said at the beginning that the tampering budget of the adversary is $\rho(n) \cdot n$. Therefore, it should be that
\begin{equation*}\rho(n) \cdot n< t \cdot (\sqrt{n \cdot \ln m}), \end{equation*}
as otherwise the evasion-robust PAC learner is not actually robust as stated. Thus, we get
\begin{equation*} m \geq e^{\rho(n) ^2  \cdot n/t}  =2^{\Omega(\rho(n) ^2  \cdot n)}\end{equation*}
which finishes the proof of Lemma \ref{lem:main}.
\end{proof}

We now prove Theorem \ref{thm:main} using Lemma \ref{lem:main}.

\begin{proof}[Proof of Theorem~\ref{thm:main}]
Using Lemma \ref{lem:main}, we will first prove Part~\ref{part:exp}, then Part~\ref{part:superpoly}, and then Part~\ref{part:ruleout}.
Throughout, $\eps=0.9,\delta=0.49$ are fixed,  so the sample complexity $m=m(n)$ is a function of $n$.

 \paragraph{Proving Part~\ref{part:exp}.}
We claim that PAC learning resisting all $b=o(n)$-tampering attacks requires sample complexity  $m \geq 2^{\Omega(n)}$. The reason is that, otherwise, there will be an infinite sequence of values  $n_1<n_2<\dots$ for $n$ for which $m=m(n_i) \leq 2^{\gamma(n_i) \cdot (n_i)}$ for $\gamma(n) =o(1)$. However, in that case, if we let $\rho(n) = \gamma(n)^{1/3}$, because $\rho(n)=o(n)$, by Lemma~\ref{lem:main}, the sample complexity is
\begin{equation*}m(n_i) \geq  2^{\Omega(\rho(n_i)^2 \cdot  n_i)} = \omega\big(2^{\gamma(n_i) \cdot  n_i}\big).\end{equation*}
However, this is a contradiction as we previously assumed $m(n_i) \leq 2^{\gamma(n_i) \cdot (n_i)}$.

\paragraph{Proving Part~\ref{part:superpoly}.}  Suppose the adversary can tamper instances with budget $b(n)=\kappa(n) \cdot \sqrt{n}$ for $\kappa(n) \in \polylog(n)$. Since we can rewrite $b(n)=\rho(n) \cdot n$ for $\rho(n)=\kappa(n)/\sqrt{n}$, then by Lemma~\ref{lem:main}, the sample complexity of $L$ should be at least 
\begin{equation*}m(n) \geq 2^{\Omega(\rho(n) ^2  \cdot n)} = 2^{\Omega(\kappa(n) ^2)}.\end{equation*}
Therefore, if we choose $\kappa(n)=\log(n)^2$, the sample complexity of $L$ becomes  $m \geq n^{\log n} \geq n^{\omega(1)}$.

\paragraph{Proving Part~\ref{part:ruleout}.} Let be $c_1,c_2\in\C$ be such that for $\Delta(c_1,c_2)  = \set{x\in \XX \mid c_1(x) \neq c_2(x)}$ we have
\begin{equation*}  \Omega(\lambda)  \leq \Pr_{x \gets D(c_1,c_2)}[x \in \Delta(c_1,c_2)] \leq \lambda.\end{equation*}
Consider a poisoning attacker $\adv_1$ that given a data set $\cS$, it removes any $(x,y)$ from $\cS$ such that $x \in \Delta(c_1,c_2)$. Note that the (expected) number of such examples is $\Pr[x \in \Delta(c_1,c_2)] \leq \lambda$. Let $\mal{\cS}$ be the modified training set. The learner $L(\mal{\cS})$ now has now way to distinguish between $c_1$ and $c_2$. Thus, like in Lemma~\ref{lem:main}, we can fix  $c \in  \set{c_1,c_2}$, such that $L(\mal{\cS})$ always produces $h$ where
\begin{equation*}\Risk(h,c)=\Pr_{x \gets D}[h(x) \neq c(x)] \geq \frac{1}{2} \cdot  \Pr_{x \gets D}[x \in \Delta(c_1,c_2)]   \geq \Omega(\lambda).\end{equation*}

For this fixed $c$ and any such learned hypothesis $h$ with $\Risk(h,c) =\Omega(\lambda)$, by Lemma~\ref{lem:advrisk}, the adversarial risk (under attacks) reaches
$\AdvRisk_{\cA_b}(h,c)\geq 0.99 $ by an attack $\adv \in \cA_b$ that changes test instances $x$ by at most $b$ for 
\begin{equation*}b = O(\sqrt{n}) \cdot \big(\sqrt {\ln(O(1/\lambda))} + \sqrt{O(1)}\big) \leq  O(\sqrt{n \cdot \ln (1/\lambda)}).\end{equation*}
Since $\lambda = 1/\poly(n)$, it holds that $b=   \Otilde(\sqrt{n})$.
\end{proof}

%% file: extensions.tex
\section{Extensions}

In this section, we describe some extensions to Theorem~\ref{thm:main} in various directions.

\paragraph{Extension to randomized predictors.}
In Theorem~\ref{thm:main}, we  ruled out PAC learning (or its small sample complexity) even for very large values $\eps=0.9,\delta=0.49$. One might argue that proving such lower bound could not be impossible because a trivial hypothesis (for the setting where $\Y=\bits$) can achieve $\eps=0.5$ by outputting random bits. However,   this trivial predictor is   \emph{randomized}, while Theorem~\ref{thm:main} is proved for deterministic hypotheses. For the case of randomized hypotheses, one can adjust the proof of Theorem~\ref{thm:main} to get similar lower bounds for $\eps=0.49,\delta=0.49$ as follows.

In the proof of Theorem~\ref{thm:main} we first showed that small sample complexity implies the existence of $c$ that with probability $>0.49$ it will have an error region with a non-negligible measure. When the hypothesis is randomized, however, we cannot work with the traditional notion of error region, because on every point $x\in \X$, the hypothesis could be wrong $h(x)\neq c(x)$ with some probability in $[0,1]$. We can, however, work with the relaxed notion of ``approximate error'' region, defined as  $\AER(h,c)=\set{x \mid \Pr_{h}[h(x) \neq c(x)] \geq 1/2} $,  where the probability is over the randomness of $h$. 

In proofs of both Lemma~\ref{lem:main} and Theorem~\ref{thm:main} we deal with two  close concept functions $c_1,c_2$ that are ``indistinguishable'' for the hypothesis $h$ and then conclude that for each point $x \in \Delta(c_1,c_2)$, $h$ makes a mistake on at least one of $c_1,c_2$.  If $h$ is randomized, we cannot say this anymore, but we can still say that for each such point $x \in \Delta(c_1,c_2)$, for at least one of $c_1,c_2$, $h(x)$ is wrong with probability at least $0.5$. Therefore, we get the same lower bound on the size of the $\AER$ as we got in Lemma~\ref{lem:main} and Theorem~\ref{thm:main}. However, expanding the set $\AER$ instead of an actual error-region, implies that the adversarially perturbed points $\mal{x}$ that fall into $\AER$ are now misclassified with probability $0.5$. Thus, at least $0.99$ fraction of inputs can be perturbed into $\AER$ to be misclassified with probability $> 0.49$.

\paragraph{Lower bound for PAC learning of a ``typical'' concept function.} Theorem~\ref{thm:main} only proves the \emph{existence} of at least \emph{one} concept function $c \in \C$ for which the (presumed) robust PAC learner will either fail (to PAC learn) or will need large sample complexity. Now, suppose concept functions themselves come from a (natural) distribution and we  only want to robustly PAC learn \emph{most} of them.  Indeed, we can extend the proof of Theorem~\ref{thm:main} to show that for natural settings, the impossibility result extends to at least \emph{half} of the concept functions, not just a few pathological cases.

To extend Theorem \ref{thm:main} to the more general ``typical'' failure over $c \gets \C$ (stated as Claim \ref{clm:ext} below) we need the following definition  as an  extension to Definition \ref{def:close}.
\begin{definition}[Uniformly $\alpha$-close function families] \label{def:closeUnif} Suppose $D$ is a distribution over $\X$, and let $\C$ be a set of functions  from $\X$ to some set $\Y$.  We call $\C$ \emph{uniformly $\alpha$-close} with respect to $D$, if there is a joint distribution $(\bfc_1,\bfc_2)$ where both coordinates are \emph{uniformly} distributed over $\C$, and that for all $(c_1,c_2) \gets (\bfc_1,\bfc_2)$, it both holds that   $c_1 , c_2 \in \C$ and  that   $\Pr_{x \gets D}[c_1(x) \neq c_2(x)]=\alpha$.
\end{definition}

\begin{claim} \label{clm:ext}
In Theorem \ref{thm:main} and Lemma \ref{lem:main}, make the only change in the setting as follows. The concept class $\C$ now satisfies the stronger condition of being \emph{uniform} $\alpha$-close with respect to $D$. Then, the same limitations of PAC learning hold for at least measure half of $c \gets \C$.
\end{claim}

Here we sketch why Claim \ref{clm:ext} holds. The difference is that now, instead of knowing the \emph{existence} of an $\alpha$-close pair $(c_1,c_2)$, we have \emph{distribution} $(\bfc_1,\bfc_2)$ samples from which satisfy the $\alpha$-close property. Therefore, for all samples $(c_1,c_2) \gets (\bfc_1,\bfc_2)$, at least one of $c_1$ or $c_2$ is ``bad'' for the (presumed) PAC learner $L$ (with the same proof before). But, since each of the coordinates in $(\bfc_1,\bfc_2)$ is marginally uniform, therefore, at least measure $1/2$ of $c \gets \C$ is bad for $L$.

\paragraph{Example.} Consider the uniform measure over homogeneous half spaces in dimension $n$ as the set of concept functions $\C$: choose a point $w$ in the unit sphere and select the half space $\set{x \mid \angles{x,w}\geq 0}$.  It is easy to see that $\C$ with such measure is uniformly $\alpha$-close with respect  to the isotropic Gaussian distribution (or uniform distribution over the unit sphere).  Thus, Claim \ref{clm:ext} applies to this case.

%\paragraph{Using known sample complexity lower~bounds.} Theorem~\ref{thm:main} is stated based on the notion of close pair of functions (with respect to distributions). This is one special technique for deriving sample complexity lower bounds for (normal) PAC learning. Therefore, using stronger distribution-specific sample complexity lower bounds for the no-attack setting  (e.g., see \cite{long1995sample,balcan2013active,sabato2013distribution}) one can get quantitatively better lower bounds. In this work, however, we focus on the asymptotic lower bounds and leave out such tighter results for future work.

%\paragraph{Lower bound for PAC learning in real-RAM model.}

%\paragraph{Lower bounds for other concentrated spaces.}

%% file: Conclusion.tex
\section{Conclusion and Open Questions}

%In this paper 
%We complement the picture of adversarial machine learning by initiating a formal study of PAC learning under adversarial perturbations such that the goal of the adversary is to increase the adversarial risk that is based on the error-region.
%and as a consequence guarantee misclassification of the perturbed inputs in every case. 
%In this direction 
We examined evasion attacks, where the adversary can perturb instances during test time, as well as hybrid attacks where the adversary can perturb instances  during  both training and test time. 
%In the case of 
For evasion attacks we gave an exponential lower bound on the sample complexity even when the adversary can perturb instances by an amount of $o(n)$, where $n$ is the data dimension capturing the ``typical'' norm of an input. 
%In the case of 
For hybrid attacks, PAC learning is ruled out altogether when the adversary can poison a small fraction of the training examples and still perturb the test instance by a sublinear amount $o(n)$ (or even $\Otilde(\sqrt n)$).

Our result shows a different behavior when it comes to PAC learning for error-region adversarial risk compared to previously used notions of adversarial robustness based on corrupted inputs. In particular, in the error-region variant of adversarial risk, realizable problems stay realizable, as normal risk zero for a hypothesis $h$ also implies (error-region) adversarial risk zero for the same $h$. This makes our results  more striking, as they apply to agnostic learning as well.

\paragraph{Open questions.} Our Theorem \ref{thm:main} relies on a level of tampering to be  at least $\Otilde( \sqrt {n})$ to imply the super-polynomial lower bounds. One natural question is to find the exact threshold of perturbations needed that triggers super-polynomial lower bounds on sample complexity. 

Another important direction is to study the sample complexity of PAC learning (with concrete parameters $\eps,\delta$) for practical distributions such as images or voice. Our lower bounds of this work are only proved for theoretically natural distributions that are provably concentrated in high dimension.  \citet{mahloujifar2019safeml},  presents a method for empirically approximating
the concentration of such distributions given i.i.d. samples from them.

Finally, we ask if similar results could be proved for corrupted-input adversarial risk. Note that  previous work studying learning under corrupted-input adversarial risk \citep{bubeck2018adversarial,cullina2018pac,feige2018robust,attias2018improved,khim2018adversarial,yin2018rademacher,montasser2019vc} focus on agnostic learning, by aiming to get close to the ``best'' robust classifier. However, it is not clear how good the best classifier is. It  remains open to find out when we can learn robust classifiers (under corrupted-input risk) in which the \emph{total} adversarial risk is small.

%Natural questions for future work are about positive results that popular learners can achieve by increasing the sample size used for training so that they can resist certain amount of perturbations under evasion and hybrid attacks. Such results make sense both for distribution-free as well as for distribution-specific settings. 